\documentclass[letterpaper, 10 pt, conference]{ieeeconf}

\IEEEoverridecommandlockouts

\usepackage{}
\usepackage{cite}

\usepackage{amsthm}
\usepackage{amssymb}
\usepackage{amsfonts,amsmath}
\usepackage{mathtools}
\usepackage[font=footnotesize]{caption}
\usepackage{subcaption}
\usepackage{xcolor}
\usepackage{graphics}

\usepackage{graphicx}
\usepackage{tikz}
\usepackage{bm}
\newlength\imagewidth
\newlength\imagescale
\usepackage{booktabs}
\usepackage{color, soul}
\usepackage{url}
\let\labelindent\relax 
\usepackage{enumitem}
\usepackage{hyperref}
\usepackage{wrapfig}

\usepackage{mdframed}

\usepackage{setspace} 
\makeatletter 
\pretocmd\@bibitem{\color{black}\csname keycolor#1\endcsname}{}{\fail}

\newcommand\citecolor[1]{\@namedef{keycolor#1}{\color{black}}}

\makeatother

\citecolor{cai2021reinforcement}  
\citecolor{fuggitti-ltlf2dfa}

\usepackage{mathrsfs}
\newtheorem{definition}{Definition}

\usepackage{xspace} 
\usepackage[bottom, multiple]{footmisc}

\usepackage{caption}

\usepackage{accents}
\usepackage{lipsum}

\usepackage[linesnumbered,vlined,ruled]{algorithm2e}
\usepackage{multirow} 
\usepackage{array,tabularx} 

\usepackage{romannum}

\usepackage{algorithm}
\usepackage[linesnumbered,ruled,vlined]{algorithm2e}

\usepackage{amssymb}
\usepackage{algpseudocode}

\newtheorem{ass}{Assumption}
\newtheorem{problem}{Problem}

\newtheorem{lemma}{Lemma}
\newtheorem{theorem}{Theorem}
\newtheorem{corollary}{Corollary}

\theoremstyle{definition}

\newcommand{\continuation}{??}

\DeclareMathOperator*{\argmax}{arg\,max}

\usepackage{float}
\title{\LARGE \bf  Motion Planning Under Temporal Logic Specifications In Semantically Unknown Environments}

\author{Azizollah Taheri and Derya Aksaray
\thanks{A. Taheri and D. Aksaray are with the Department of Electrical and Computer Engineering, Northeastern University, Boston, MA, 02115.}
\thanks{This work was partially supported by ARL contract W911NF2420026.}
}

\IEEEoverridecommandlockouts
\begin{document}
\maketitle
\thispagestyle{empty}
\pagestyle{empty}

\begin{abstract}
This paper addresses a motion planning problem to achieve spatio-temporal-logical tasks, expressed by syntactically co-safe linear temporal logic specifications {($\text{scLTL}_{\setminus{\emph{next}}}$)}, in uncertain environments. Here, the uncertainty is modeled as some probabilistic knowledge on the semantic labels of the environment. For example, the task is ``first go to region 1, then go to region 2"; however, the exact locations of regions 1 and 2 are not known a priori, instead a probabilistic belief is available. We propose a novel automata-theoretic approach, where a special product automaton is constructed to capture the uncertainty related to semantic labels, and a reward function is designed for each edge of this product automaton. The proposed algorithm utilizes value iteration for online replanning. We show some theoretical results and present some simulations/experiments to demonstrate the efficacy of the proposed approach. 

\end{abstract}

\color{black}

\section{Introduction}
Autonomous navigation is essential for a wide range of applications from domestic robots to search and rescue missions\cite{lavalle2006planning}. Traditional motion planning focuses on generating robot trajectories that navigate from an initial state to a desired goal while avoiding obstacles. However, future's applications demand solutions to more complex tasks beyond simple point-to-point navigation. These tasks, such as sequencing, coverage, response, and persistent surveillance, can be expressed using temporal logic (TL), 
which provides a structured and compact way to define high-level mission requirements (e.g., \cite{plaku2012planning,lahijanian2011temporal}). The integration of TL in motion planning has led to the development of verifiable control synthesis methods that enable robots to satisfy desired TL constraints (e.g., \cite{fainekos2005hybrid}).

Existing planning algorithms under TL constraints often assume complete knowledge about the environment, which allow for the design of correct-by-construction controllers (e.g.,\cite{kress2007s,vasile2013sampling}). However, in real-world scenarios, robots often operate in partially or fully unknown environments, which requires the ability to adapt and replan as new information becomes available. For example, Fig. \ref{fig:exp} illustrates a scenario where a drone aims to first \color{black} visit \color{black} Region 1 then Region 2 while it doesn't exactly know where these regions are. To address these challenges, recent studies have focused on incorporating semantic uncertainties into the planning process (e.g., \cite{guo2013revising, kantaros2020reactive,haesaert2018temporal,hashimoto2022collaborative,kantaros2022perception, kloetzer2012ltl, 8272366}). This involves considering label uncertainties in the robot's environment, e.g., unknown regions and obstacles. 

In this paper, we introduce an automata-theoretic framework that addresses motion planning problems for a sub-class of Linear Temporal Logic (LTL) specifications in uncertain environments. The proposed framework leverages value iteration algorithm to compute the control policy based on the current probabilistic belief of the labels in the environment. We show that the desired co-safe LTL specification is satisfiable as long as there exists a way to satisfy it based on the initial belief.

\begin{figure}[t!]
    \centering
    \begin{tabular}{cc}
        \begin{subfigure}[t]{0.4\columnwidth}
            \centering
            \includegraphics[width=\textwidth]{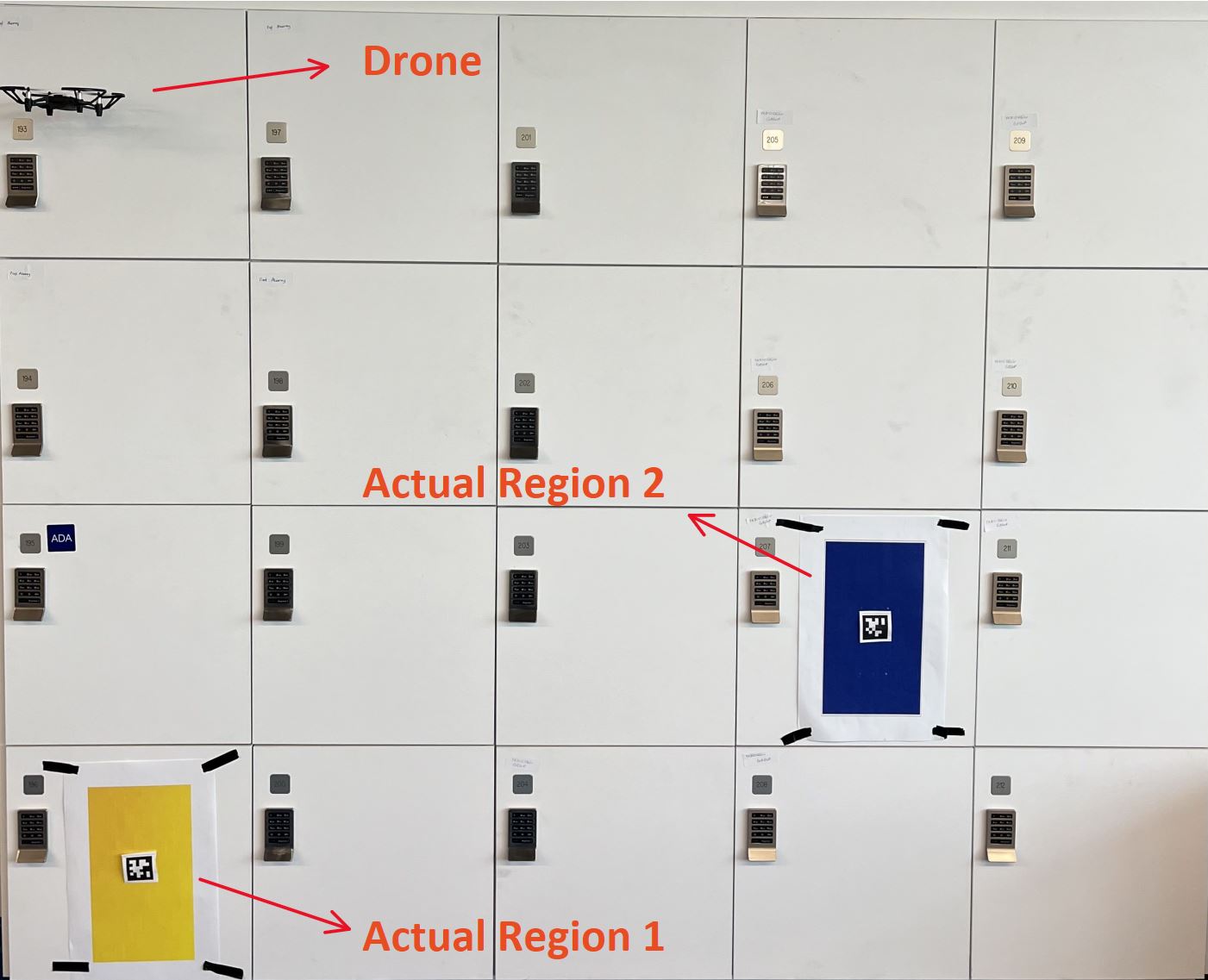}
            \caption{}  
        \end{subfigure} &
        \begin{subfigure}[t]{0.4\columnwidth}
            \centering
            \includegraphics[width=\textwidth]{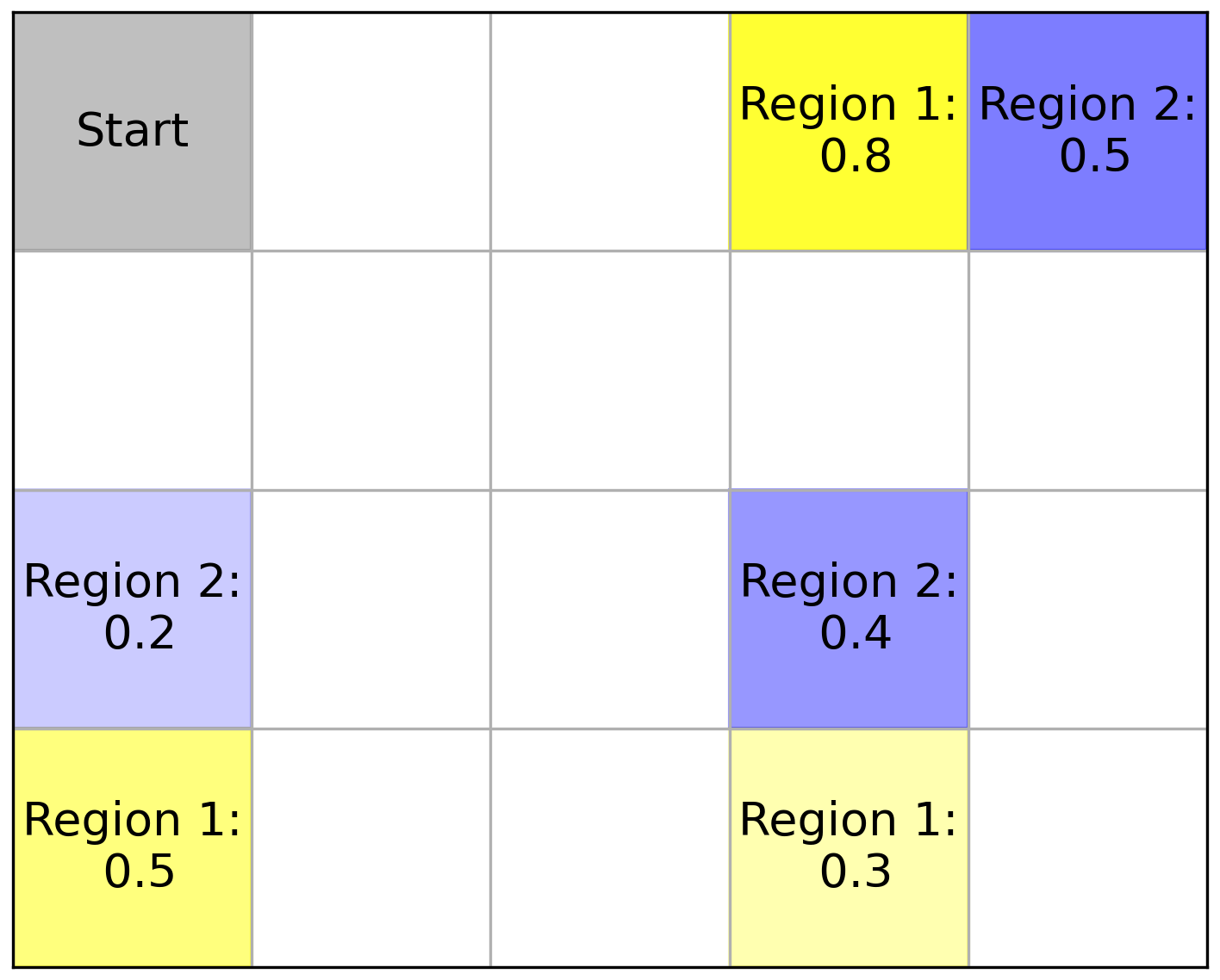}
            \caption{}  
        \end{subfigure}
    \end{tabular}
    \caption{\scriptsize (a) A drone in an environment aims to \color{black} visit \color{black} Region 1 first, and then Region 2, without prior knowledge of their exact locations. (b) The drone's initial belief about the possible locations of regions 1 and 2. \textcolor{black}{ The values represent independent probabilities for the cells, indicating the likelihood of a cell having that label (or no label otherwise).}}
    \label{fig:exp}
\end{figure}

\section{Related Work}
The works that are related to our paper are \cite{guo2013revising, kantaros2020reactive,haesaert2018temporal,hashimoto2022collaborative,kantaros2022perception, kloetzer2012ltl, 8272366}, where they also consider label uncertainty in the environment. However, the studies in\cite{guo2013revising} and\cite{kantaros2020reactive} do not consider probabilities over the semantic map. Instead, the authors in \cite{guo2013revising} construct an initial product graph (incorporating the physical state and task progress), and they revise it (pruning/adding edges) after obtaining new information about the unknown map. Alternatively, the authors in \cite{kantaros2020reactive} consider known regions of interest and unknown but static obstacles, whereas we consider a probability distribution for the labels associated with regions of interest. 

\color{black}
Similar to our paper, the works  \cite{haesaert2018temporal,hashimoto2022collaborative,kantaros2022perception,kloetzer2012ltl,8272366} consider probabilities over the labels in the environment. In \cite{haesaert2018temporal}, a planning problem is formulated as a partially observable Markov decision process (POMDP), where the robot is uncertain about both its current position and the location of target regions. In contrast, we assume that the robot has perfect knowledge of its current position and receives accurate observations at that location.
In \cite{hashimoto2022collaborative}, the authors address planning in an uncertain semantic map using a team consisting of a copter and a rover. The rover is responsible for satisfying an scLTL mission specification, while the copter assists by first exploring the environment and reducing uncertainty. In this setup, the copter can move freely without risking task violation, and the rover's safety relies on the information gathered by the copter. In contrast, we consider a single robot which individually ensures both safety and task satisfaction guarantees. In \cite{kantaros2022perception}, the authors propose a sampling-based approach, where probabilistic labels are transformed from probabilistic to certain using some user-defined thresholds. Here threshold-sensitive label assignments ignore states below the threshold and treat those above it as identical, which causes planning inefficiencies.
In contrast, our method considers all label uncertainties collectively during planning. In \cite{kloetzer2012ltl} and \cite{8272366}, the authors present an automata-theoretic method, which solves the planning problem by incorporating all probabilistic labels in the environment. Nonetheless, both \cite{kloetzer2012ltl} and \cite{8272366} consider a model for the probabilistic labels, which appear and disappear in some known states with some unknown frequency. 
As a result, their solutions rely on the fact that the robot stays in those states until the desired label appears. However, this approach can lead to a deadlock, where the robot waits indefinitely under the false assumption that a label exists in a state, even though it does not exist in reality. In contrast, our paper can address such scenarios where the robot may hold incorrect beliefs about the labels of the states. Moreover, our paper proposes an online approach which enables the robot to replan whenever new information is discovered whereas \cite{8272366} presents an offline approach where a linear program, that takes into account all probabilistic labels, is solved in one shot before the mission. Hence, \cite{8272366} may get stuck due to incorrect initial beliefs (as discussed in our benchmark analysis in Sec. VI) whereas our paper facilitates resilient planning in semantically uncertain environments. 



\color{black}
\vspace{-1mm}

\section{Preliminaries}
\subsection{Notation and Graph Theory}
\color{black}
Let $O$ be a set of Boolean statements. The power set (set of all subsets) of $O$ is denoted by $2^O$. The set of infinite sequences (words) defined over $O$ and $2^O$ are denoted by $O^\omega$ and $(2^O)^\omega$, respectively. The set of positive real numbers is denoted by $\mathbb{R}^{>0}$.

A directed graph is defined as a tuple $\mathcal{G} = (X, \Delta)$, where X represents the nodes and $\Delta \subset X \times X$ is a set of directed edges connecting these nodes. A node $x_j$ is considered an out-neighbor of another node $x_i$ if $(x_i, x_j) \in \Delta$. We use $N_{x_i}$ to represent the set of out-neighbors of $x_i$. In this paper, for brevity, we use the term ``neighbor" instead of ``out-neighbor". $N_{x_i}^h$ denotes the set of all nodes that can be reached from $x_i$ within at most $h$-hops.
\color{black}


\subsection{Probabilistically Labeled DMDP}
We define a Probabilistically Labeled Deterministic Markov Decision Process (PL-DMDP) as in \cite{8272366} but more compact, to model semantic uncertainty in an environment.
\begin{definition}[PL-DMDP]
\textit{ A probabilistically labeled deterministic MDP (PL-DMDP) is a tuple $\mathcal{M} = (X, \Sigma, \delta, O, L,l_{true}, p_L, c)$, where:
\begin{itemize}
    \item   $X$ is the finite set of states,
    \item   $\Sigma$ is a finite set of actions,
    \item   $\delta: X \times \Sigma \rightarrow X$ is a deterministic transition function,
    \item $O$ is the set of observations,
    \item $L : X \rightarrow 2^{2^O}$ is the labeling function,
    \item $l_{true}: X \rightarrow 2^O$ is the true label function,
    \item $p_L: X \times 2^O \rightarrow [0,1] $ is a mapping such that $p_L(x, l)$ indicates  the probability \textcolor{black}{ of seeing the set of observations $  l \in 2^O$ in state $x\in X$.} Note that for any $x\in X$,  $\sum_{  l\in L(x)}p_L(x,   l) = 1$ and if $  l\notin L(x)$ then $p_L(x,  l) = 0$,
    \textcolor{black}{\item $c: X \times X \rightarrow  \mathbb{R}^{>0}$ is the cost function.}
\end{itemize}
}
\end{definition}

An example of a PL-DMDP, where the knowledge about the label of each state is uncertain, is shown in \textcolor{black}{Fig. \ref{fig:figure2}}. Note that the standard labeled deterministic MDP (e.g.,\cite{9636598,10342259}) $\mathcal{M}_l = (X, \Sigma, \delta, O, L^*, c)$ is a special case of PL-DMDP, where $L^*: X \rightarrow 2^O$ is the labeling function and there is no non-determinism in labeling the states. When the label of a state is revealed, the $p_L$ function may change as new information is discovered, which makes both $p_L$ and the PL-DMDP time-varying. For simplicity, we denote them as $p_L$ and $\mathcal{M}$.

\color{black}
Given a PL-DMDP $\mathcal{M}$,  a finite \textbf{action sequence} is   $\color{black}\bm{\sigma{\scriptstyle[0:n]}} = \sigma(0)\sigma(1)\dots\sigma(n)$, where $\sigma(i)\in \Sigma$ for all $i \in\{0,1,\dots,n\}$;
\color{black} a finite \textbf{trajectory} generated by  $\bm{\sigma{\scriptstyle[0:n]}}$ is  $\color{black}\bm{x^{\sigma{\tiny[0:n]}}}=x(0)x(1)\dots x(n+1)$, where $x(i) \in X$; and the corresponding \textbf{word} is $\bm{  l({\color{black}\color{black}\bm{x^{\sigma{\tiny[0:n]}}}  })} =   l(0)  l(1) \dots   l(n+1)$, where $  l(i) = l_{true}(x(i))$. We define the cost of the  trajectory $\bm{x^{\sigma{\tiny[0:n]}}}$  as $C(\color{black}\color{black}\bm{x^{\sigma{\tiny[0:n]}}}  ) = \sum_{i=0}^{n}  c(x(i),x(i+1)) $. \textcolor{black}{In this paper, we consider a uniform cost function 
$c$, which can be interpreted as the time required to complete the mission under the assumption that each transition takes equal time.}
\color{black}

\subsection{Temporal logic}
\begin{wrapfigure}{r}{0.48\columnwidth}
    \centering
     \vspace{-10mm}
\centerline{\includegraphics[width=0.48\columnwidth]{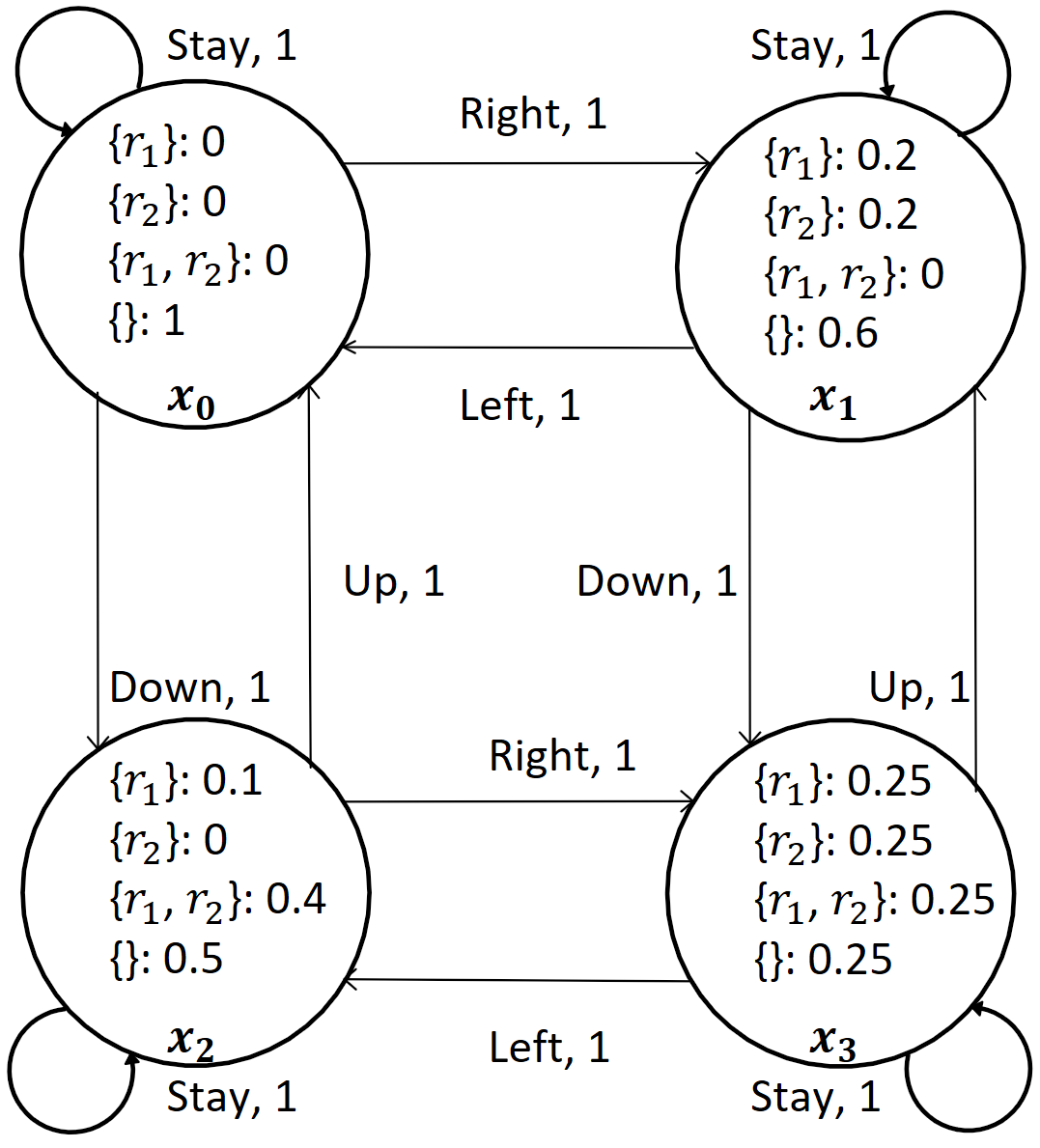}}
\caption{\scriptsize An illustration of PL-DMDP $\mathcal{M}_1 = (X, \Sigma, \delta, O, L, l_{true}, p_L, c)$ where $X=\{x_0, x_1, x_2, x_3\}$, $\Sigma=\{Up,Right,Down,Left,Stay\}$, and $O=\{r_1,r_2\}$. Some examples of $\delta$, $L$, and $p_L$ are $\delta(x_0, Right)=x_1$, $L(x_0) = \{{\{r_1\},\{r_2\},\{r_1,r_2\},\{\}}\}$, $p_L(x_1, \{r_1\}) = 0.2$, respectively.}
 \vspace{-5mm}
\label{fig:figure2}
\end{wrapfigure}
Temporal logic (TL) is a formal language used to define the temporal characteristics of a dynamical system. Linear Temporal Logic (LTL) is a type of TL that can handle words of infinite length $\bm{l} =   l(0)  l(1)  l(2)\dots$ where $  l(i)\in 2^O$ for all $i\geq0 $. LTL is widely employed in diverse domains (e.g.,\cite{4459804,fu2016optimal,guo2013revising,buyukkoccak2023state }) and can be used for verification and control synthesis in complex missions. \textcolor{black}{In this work, we focus on a specific fragment of LTL known as scLTL}.
\begin{definition}[$\text{scLTL}_{\setminus{\emph{next}}}$]   
\textit{A syntactically co-safe linear temporal logic ($\text{scLTL}_{\setminus{\emph{next}}}$) formula $\phi$ over a set of observations \textit{O} is recursively defined
as:
\begin{equation*}
\phi = \top\mid o\mid \neg o\mid \phi_1\lor\phi_2\mid\phi_1\land\phi_2\mid\phi_1\mathcal{U}\phi_2\mid\color{black}\lozenge\phi_1\color{black}
\end{equation*}
where $o\in O$ is an observation and $\phi$, $\phi_1$ and $\phi_2$ are $\text{scLTL}_{\setminus{\emph{next}}}$ formulae.
$\top$ (true), $\neg$ (negation), $\lor$ (disjunction) and $\land$ (conjunction) are Boolean operators, and $\mathcal{U}$ (until) and $\lozenge$ (eventually) are temporal operators.}
\end{definition}
The \emph{globally} operator cannot be represented in this language, since $\text{scLTL}_{\setminus{\emph{next}}}$ only allows for the negation of the observations. This means that expressions like $\neg \lozenge \neg \phi$ are not part of the $\text{scLTL}_{\setminus{\emph{next}}}$ fragment. \color{black}
We exclude the \emph{next} operator from the syntax (e.g., \cite{kantaros2020reactive}) because requiring tasks, like finding an object in $n$ steps when its location is unknown, is overly restrictive in uncertain environments. In the remainder of the paper, we will refer to $\text{scLTL}_{\setminus{\emph{next}}}$ as scLTL.
 
 \color{black}
 The semantics of scLTL formulae are interpreted over infinite words in $2^O$. We define \color{black}the \color{black} language of an scLTL formula $\phi$ as the set of infinite words satisfying $\phi$  and denote it as $\mathscr{L}_\phi
$. Even though scLTL formulae are interpreted over infinite words (i.e., over $({2^O})^\omega$), their satisfaction is guaranteed in finite time.
Any infinite word $\bm{  l} =   l(0)  l(1)  l(2)\dots$ that satisfies formula $\phi$ contains a finite “good”
prefix $  l(0))  l(1) \dots   l(n)$ such that any infinite word that contains the prefix, i.e.,
$  l(0)  l(1) \dots   l(n)  \bm{l^\prime} ,   \bm{l^\prime} \in (2^O)^\omega$
, satisfies $\phi$ \textcolor{black}{\cite{belta2017formal}}. We denote the language of
finite good prefixes of $\phi$ by $\mathscr{L}_{pref,\phi}$.
A deterministic finite state automaton (DFA) can be constructed from any scLTL formula\color{black}\cite{fuggitti-ltlf2dfa} \color{black} that compactly represents all the satisfactory words and defined as follows:

\begin{definition}[DFA] 
\textit{A deterministic finite state automaton (DFA) is a tuple $\mathcal{A} = (S, s_0, 2^O, \delta_{\color{black}a\color{black}}, F_a)$, where:
\begin{itemize}
\item $S$ is a finite set of states,
\item $s_0 \in S$ is the initial state,
\item $2^O$ is the alphabet,
\item $\delta_{\color{black}a\color{black}}: S \times 2^O \rightarrow S$  is a transition function,
\item $F_a \subseteq S$  is the set of accepting states.
\end{itemize}
}
\label{DFA}
\end{definition}

The semantics of a DFA are defined over finite input words in $2^O$. A run of DFA $\mathcal{A}$ over a word $\bm{  l} =  l(0)  l(1) \dots   l(n)$ is represented by a sequence $\bm{s} = s(0)s(1)\dots s(n+1)$ where $s(i)\in S$, $s(0) = s_0$ and $s(i+1) = \delta_{\color{black}a\color{black}}(s(i),  l(i))$ for all $i \geq 0$. If the corresponding run of the word $\bm{  l}$ ends in an accepting state, i.e., $s(n+1) \in F_a$, then we say the word is accepted. The language accepted by $\mathcal{A}$ is the set of all words accepted by $\mathcal{A}$ and is denoted by $\mathscr{L}_{\color{black}\mathcal{A}\color{black}}$.
An scLTL formula $\phi$ over a set $O$ can always be translated into a DFA $\mathcal{A}_\phi$ with alphabet $2^O$ that accepts all and only good prefixes of $\phi$ (i.e., $\mathscr{L}_{\mathcal{A}_\phi}
= \mathscr{L}_{pref,\phi}$ ) \cite{belta2017formal}. Note that the DFA described above only
encodes accepting words. However, one can construct a total
DFA in order to track the violation cases as well.


\begin{definition}[Total DFA]
\textit{A DFA is called total if for all $s \in S$ and any $  l \in 2^O$, the transition $\delta_{\color{black}a\color{black}}(s,   l )\neq \emptyset$ \cite{lin2015hybrid}.}
\end{definition}

For any given DFA $\mathcal{A}$, 
one can always create a language-equivalent \color{black}(also defined over finite input words) \color{black} total DFA by adding a trash state, referred to as $s_t$ and introducing a transition $\delta_{\color{black}a\color{black}}(s, l)= s_t$ if and only if  $\delta_{\color{black}a\color{black}}(s, l)= \emptyset$. 

\section{Problem Formulation}\label{sec:problem_formulation}
We consider a robot whose goal is to achieve an scLTL task $\phi$ in an environment where the semantic labels are static but initially unknown. \textcolor{black}{Accordingly, we model the robot's decision-making as a PL-DMDP $\mathcal{M} = (X, \Sigma, \delta, O, L,l_{true}, p_L, c)$ with unknown $l_{true}$} and a uniform cost function $c$, which has a value of $\beta > 0$.
\color{black} A standard way of formulating such scLTL planning problems (e.g., \cite{kantaros2020reactive}, \cite{kantaros2022perception}) is as follows:\color{black}
    



 
\vspace{-3mm}
\begin{subequations}
 \begin{align}
    & \hspace{-5mm}\min_{n,\color{black} \color{black} \color{black} \color{black} \bm{\sigma{\scriptstyle[0:n]}} \color{black} \color{black} \color{black} \color{black}}[ C(\color{black}\color{black}\bm{x^{\sigma{\tiny[0:n]}}} \color{black} \color{black})] \\
    &  x(0) = x_0, \sigma(i) \in \Sigma, \\
    & x(i+1) = \delta(x(i),\sigma(i)), \\
    & \bm{  l({\color{black}\color{black}\bm{x^{\sigma{\tiny[0:n]}}} \color{black} \color{black}})}\in \mathscr{L}_{pref,\phi},
\end{align}
\end{subequations}
 \noindent where the constraints (1b) and (1c) hold for all $i \in\{0,1,\dots,n\}$.
 Note that when $l_{true}$ is unknown (as in our case), the constraint (1d) $\bm{  l({\color{black}\color{black}\color{black}\bm{x^{\sigma{\tiny[0:n]}}} \color{black} \color{black} \color{black}})}$ cannot be evaluated. This requires reformulating the problem to handle initially unknown semantic labels. Hence, the next section introduces a higher resolution representation (product automaton), our problem defined over it, and our proposed solution.


\section{Solution Approach}\label{sec:solution_approach}


We propose an automata-theoretic approach, with offline and online parts, to \color{black} find a sequence of actions $\bm{\sigma{\scriptstyle[0:n]}}$ such that $\bm{  l({\bm{x^{\sigma{\tiny[0:n]}}} })}\in \mathscr{L}_{pref,\phi}$. 

\color{black}
\subsection{Product Automaton}

We construct a special  product automaton that can model the uncertainty associated with the labels.
\vspace{-2mm}
\begin{definition}[Product automaton] \textit{ Given a PL-DMDP $\mathcal{M} = (X, \Sigma, \delta, O, L, l_{true}, p_L, c)$ and a total DFA $\mathcal{A} = (S, s_0, 2^O, \delta_{\color{black}a\color{black}}, F_a, \{s_t\})$, the product automaton is a tuple $\mathcal{P}=\mathcal{M}\times \mathcal{A}=(S_p, S_{p_0},\Sigma, \delta_p, p_p, \mathcal{F}_a, \mathcal{F}_t)$, where:
\begin{itemize}
  \item $S_p = X \times S$ is the set of states,
  \item $S_{p_0} = X \times \{s_0\} \subseteq S_p$ is the set of initial states,
  \item  $\Sigma$ is the finite set of actions,
  \item $\delta_p \subseteq S_p \times S_p$ is a transition relation such that for any $(x, s)$ and $(x^\prime,s^\prime)\in S_p$, we have $((x, s), (x^\prime,s^\prime)) \in \delta_p$ if and only if $\exists \sigma \in \Sigma$ such that $\delta(x,\sigma) = x^\prime$ and $\exists   l \in L(x^\prime) \text{ such that } p_L(x^\prime,  l)\geq 0\text{ and } \delta_{\color{black}a\color{black}}(s,  l) = s^\prime $,
  \item $p_p: S_p \times \Sigma \times S_p \rightarrow [0,1]$ assigns a probability to each edge in the product automaton $\mathcal{P}$ based on the information of $\mathcal{M}$ such that $\forall ((x, s), (x^\prime,s^\prime)) \in \delta_p$ and $\sigma \in \Sigma$, $p_p((x, s),\sigma,  (x^\prime,s^\prime)) = \sum_{  l\in L_{s\rightarrow s^\prime}}p_L(x^\prime,  l)$ if $\delta(x,\sigma) = x^\prime$; otherwise $p_p((x, s),\sigma,  (x^\prime,s^\prime)) = 0$. Here, $L_{s\rightarrow s^\prime} = \{  l\in 2^O \mid \delta_{\color{black}a\color{black}}(s,  l) = s^\prime\}$,
  \item $\mathcal{F}_a = X \times F_a$ is the set of accepting states,
  \item $\mathcal{F}_t = X \times \{s_t\}$ is the set of trash states.
\end{itemize}
}
\label{def:P}
\end{definition}

The product automaton encodes both sets of physical states and the total DFA states of the robot. Reaching an accepting state in the product automaton guarantees the satisfaction of constraint (1d). Our goal will be to find a policy over this product automaton that has a non-zero probability of reaching an accepting state while minimizing the expected cost. To this end, we define the following reward function. 

\subsection{Reward Design And Value Iteration}
For any two product automaton states $s_p = (x,s)$ and $s_p^\prime=(x^\prime,s^\prime) \in S_p$, we define our reward function as follows:

\begin{equation}\label{eq:reward}
    \scriptstyle
    r(s_p,\sigma, s_p^\prime)
    =
    \small
    \begin{Bmatrix}
     \frac{-\beta}{1-\gamma}\color{black}& \text{      if } s_p \notin \mathcal{F}_t, s_p^\prime \in \mathcal{F}_t \text{ and } \delta(x,\sigma) = x^\prime  \\
     0 & \text{      if } s_p \in \{\mathcal{F}_a\cup \mathcal{F}_t\}   \\
     -\beta& \text{otherwise}
    \end{Bmatrix}.
\end{equation}

 The first expression indicates a reward of $ \frac{-\beta}{1-\gamma}$ when transitioning from a non-trash state to a trash state, which assigns a large penalty for violating the specification. Here, $\gamma\in [0,1)$ is a discount factor, 
 \textcolor{black}{and its value determines the importance given to future rewards when computing the policy as in \eqref{eq:return}.} The second expression indicates $0$ reward to the transitions that start from an accepting or a trash state. This implies that we disregard events occurring after the robot enters an accepting or a trash state. In the third expression, a uniform negative reward ($-\beta$) is applied to all other transitions, which penalizes the robot for not completing the mission. This reward shaping strategy encourages the robot to reach the set of accepting states in $\mathcal{F}_a$ while avoiding the trash states in $\mathcal{F}_t$ and minimizing the cost. An example of the defined reward function over a portion of the product automaton is illustrated in \textcolor{black}{Fig. \ref{fig:figure4}}.

\begin{wrapfigure}{r}{0.2\textwidth}
    \centering
    \includegraphics[width=0.2\textwidth]{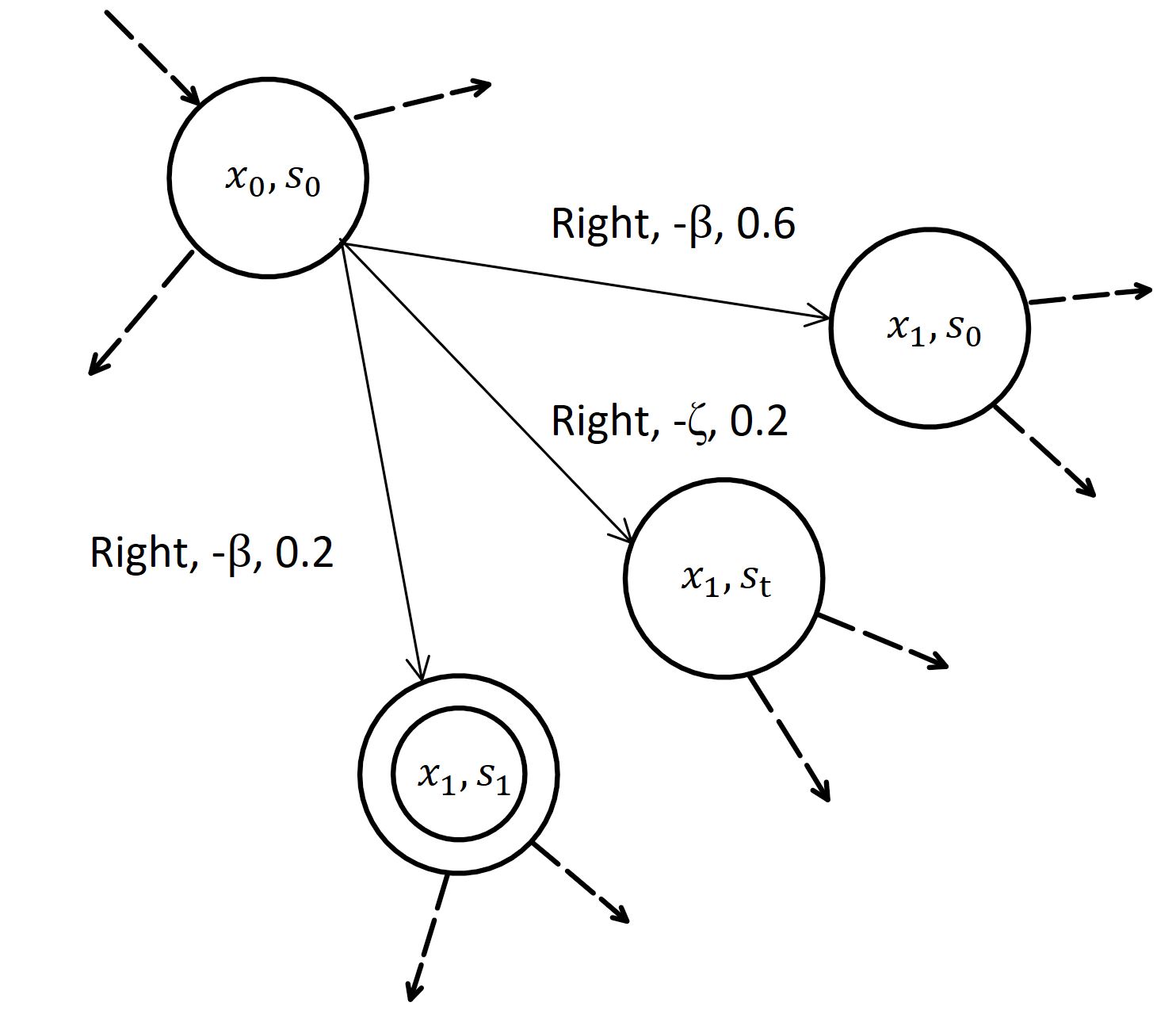}
    \caption{\scriptsize A portion of the product automaton 
    showing only the outgoing edges from $(x_0, s_0)$ under action $Right$, along with their associated probabilities and rewards.}
    \label{fig:figure4}
\end{wrapfigure}
A deterministic policy $\pi_p$ is a function that maps each product automaton state
to an action, i.e., $\pi_p: S_p \rightarrow \Sigma$. A trajectory over $\mathcal{P}$ generated by $\pi_p$ is an infinite sequence $ \bm{s_p^{\pi_p}}= s_p(0)s_p(1) \dots$, where $s_p(i)\in S_p$ for all $i\geq 0$. Due to the uncertainty associated with the labels, multiple trajectories can be generated over $\mathcal{P}$ under a policy $\pi_p$. We define the expected return of any state $s_{p}\in S_p$ under policy $\pi_p$ as follows:

\begin{equation}\label{eq:return}
\scriptsize
U^{\pi_p}(s_p) =\mathbb{E}^{\pi_p}\left[\sum_{i=0}^{\infty}\gamma^i r\big(s_{p}(i),\pi_p\big(s_p(i)\big), s_p(i+1)\big) \bigg| s_p(0)=s_p\right].
\end{equation}

\color{black} 

As the robot collects new information about the environment by revealing true labels, the value of the expected return in (\ref{eq:return}) changes accordingly. Therefore, we define the following problem, which is solved iteratively.

\color{black}
   
\color{black}\begin{problem}

 \textit{Given a product automaton $\mathcal{P}$ as in Def. \ref{def:P}, find a policy $\pi_p^*$ that maximizes the expected return $U^{\pi_p}(s_p)$ for all the states $s_p\in S_p$:}\color{black}

 \begin{equation}\label{eq:policy}
   \pi_p^*(s_p) =  \argmax_{\pi_p\in \Pi_p}[U^{\pi_p}(s_p)], \forall s_p \in S_p,
\end{equation}
\noindent where $\Pi_p$ is the set of all deterministic policies over $S_p$. 

 \vspace{-2mm}
\end{problem}\color{black}
\begin{ass} \textit{If the robot is in a state $x \in X$, it has access to \textcolor{black}{ $l_{true}(x)$ and $l_{true}(x')$ for all $x'\in N_{x}^h$.}  Here $h\geq1$ depends on the range of the robot's sensors.}
\label{sensor}
\end{ass}

Assumption \ref{sensor} is not restrictive, as most modern sensors offer high accuracy in capturing environmental details, which enables the robot to reliably determine the labels of the states. To update the belief based on sensor data, we update $p_L$ for each state $x \in N_{x_{\text{current}}}^h$ as follows:

\begin{equation}\label{eq:map_update}
    \small
    p_L(x,l)
    =
    \begin{Bmatrix}
    1& \text{      if }  l = l_{true}(  x)    \\
     0 & \text{    otherwise } 
    \end{Bmatrix}.
\end{equation}


\begin{definition}[Uncertain States]
The set of uncertain states is defined as $\mathcal{X} \subseteq X$, where $\mathcal{X}$ includes all states $x \in X$ for which the true label $l_{\text{true}}(x)$ is not known.
\end{definition}

\color{black}
\begin{definition}[Non-zero probability satisfying policy]
    \textit{A policy $\pi_p$ is a non-zero probability satisfying policy if, starting from any initial state $s_p(0) \in \mathcal{S}_p \setminus \mathcal{F}_t$, among all possible trajectories that can be generated by $\pi_p$ over $\mathcal{P}$, there exists a trajectory $\bm{s_p^{\pi_p}} = s_p(0)s_p(1) \dots$, in which there is a state $s_p(i) \in \mathcal{F}_a$ for $i \geq 0$.
}\end{definition}



\color{black}By following a non-zero probability satisfying policy, there is always a chance of satisfying the specification.

\begin{ass}
\label{non-zerp}
    At any time during the mission, there exists a non-zero probability satisfying policy over $\mathcal{P}$.
\end{ass}
Assumption \ref{non-zerp} is a mild assumption indicating the feasibility of $\phi$ over the PL-DMDP. It states that the initial belief must contain sufficient information to ensure the existence of a non-zero probability satisfying policy. According to this assumption, when $l_{\text{true}}$ is known for all states, the existence of a trajectory from any state $s_p \in S_p \setminus (\mathcal{F}_a \cup \mathcal{F}_t)$ to an accepting state is guaranteed.

\color{black}

\vspace{-1mm}

\begin{lemma}
\label{threshold}
\textit{
A policy $\pi_p$ is a non-zero probability satisfying policy if and only if it satisfies the condition $U^{\pi_p}(s_p) > \frac{-\beta}{1 - \gamma}$ for all $s_p \in \mathcal{S}_p \setminus (\mathcal{F}_a \cup \mathcal{F}_t)$.
}
\end{lemma}
\begin{proof}
   For any policy $\pi_p$, all possible trajectories that can be generated over the product automaton $\mathcal{P}$, starting from a state $s_p \in \mathcal{S}_p \setminus (\mathcal{F}_a \cup \mathcal{F}_t)$, can be categorized into three distinct groups:

    \textbf{Group 1:} Trajectories that never reach an accepting state or a trash state. For any such trajectory, the return is:
    \begin{equation}
        U_1 = \sum_{i=0}^{\infty} \gamma^i (-\beta) = \frac{-\beta}{1 - \gamma}.
    \end{equation}

    \textbf{Group 2:} Trajectories that reach a trash state. Suppose a trajectory reaches a trash state after $n$ time steps. Then, the return is:
    \begin{equation}
        U_2 = -\beta - \gamma\beta  - \dots - \gamma^{n-1}\beta - \gamma^n \frac{\beta}{1 - \gamma}
        = \frac{-\beta}{1 - \gamma}.
    \end{equation}

    \textbf{Group 3:} Trajectories that reach an accepting state. Suppose the trajectory reaches an accepting state after $n$ time steps. Then, the return is:
    \begin{equation}
        U_3 = -\beta - \gamma\beta - \dots - \gamma^{n-1}\beta + \gamma^n \cdot 0
        > \frac{-\beta}{1 - \gamma}.
    \end{equation}

    For a policy $\pi_p$ and any state $s_p \in \mathcal{S}_p \setminus (\mathcal{F}_a \cup \mathcal{F}_t)$, the expected return $U^{\pi_p}(s_p)$ is a weighted sum over the three groups:

    \begin{equation}
        \label{weights}
        U^{\pi_p}(s_p) = w_1 U_1 + w_2 U_2 + w_3 U_3,
    \end{equation}
    where $w_1$, $w_2$, and $w_3$ are the probabilities of generating trajectories in each group, and $w_1 + w_2 + w_3 = 1$.

    Now assume that $\pi_p$ is a non-zero probability satisfying policy. Then, by definition, there exists at least one trajectory that reaches an accepting state, which implies $w_3 > 0$. Since $U_3 > \frac{-\beta}{1 - \gamma}$, and $U_1 = U_2 = \frac{-\beta}{1 - \gamma}$, it follows from \eqref{weights} that:
    \[
    U^{\pi_p}(s_p) > \frac{-\beta}{1 - \gamma}.
    \]

    Conversely, suppose that for all $s_p \in \mathcal{S}_p \setminus (\mathcal{F}_a \cup \mathcal{F}_t)$, we have $U^{\pi_p}(s_p) > \frac{-\beta}{1 - \gamma}$. From \eqref{weights}, this inequality can only hold if $w_3 > 0$, since both $U_1$ and $U_2$ are equal to $\frac{-\beta}{1 - \gamma}$. Thus, at least one trajectory generated by $\pi_p$ must reach an accepting state, which implies that $\pi_p$ is a non-zero probability satisfying policy.
\end{proof}

\begin{corollary}
\label{corollary}
Let Assumption~\ref{non-zerp} hold. Then, the policy $\pi_p^*$ computed in~\eqref{eq:policy} satisfies $U^{\pi_p^*}(s_p) > \frac{-\beta}{1 - \gamma}$ for all $s_p \in \mathcal{S}_p \setminus (\mathcal{F}_a \cup \mathcal{F}_t)$. Therefore, $\pi_p^*$ is a non-zero probability satisfying policy.
\end{corollary}

\color{black}

 \vspace{-1mm}

\begin{lemma}
\label{known}
\textit{Let $\mathcal{P}=(S_p, S_{p_0}, \Sigma, \delta_p, p_p, \mathcal{F}_a, \mathcal{F}_t)$ be the product of PL-DMDP $\mathcal{M} = (X, \Sigma, \delta, O, L, l_{true}, p_L, c)$ with known $l_{true}$ \textcolor{black}{ (i.e., set of uncertain states $\mathcal{X}=\emptyset$)} and a total DFA $\mathcal{A}$ of the desired scLTL $\phi$. If Assumption \ref{non-zerp} holds, then the policy in  (\ref{eq:policy}) which maximizes the return in (\ref{eq:return}) will lead the robot to an accepting state in $\mathcal{F}_a$.}
\end{lemma}

\begin{proof}
Let us assume that the policy $\pi_p^*$ in~\eqref{eq:policy}, that maximizes the return in~\eqref{eq:return}, does not lead the robot to an accepting state. Since $l_{\text{true}}$ is known, the transitions in $\mathcal{P}$ are deterministic, and therefore $\pi_p^*$ generates a unique trajectory that does not reach any accepting state in $\mathcal{F}_a$.
This contradicts Corollary~\ref{corollary}, which states that if Assumption \ref{non-zerp} holds, $\pi_p^*$ is a non-zero probability satisfying policy, which implies the existence of at least one trajectory that reaches an accepting state. Therefore, the assumption is contradicted, and $\pi_p^*$ must lead the robot to an accepting state in $\mathcal{F}_a$.
\end{proof}

\color{black}
The policy $\pi_p^*$ can be computed by the value iteration algorithm \cite{bellman1957markovian} using the reward function in (\ref{eq:reward}). This algorithm has a computational time complexity of $\mathcal{O}(|S_p|^2 \times |\Sigma|)$. Note that lines 4-9 of Alg. 1 represent the standard iterative Bellman equation over the product automaton states. In line 8, $r$ and $p_p$ are the reward function and the probability of the edges of the product automaton, respectively. Finally, lines 10-12 find the optimal policy $\pi_p^*$ as in (\ref{eq:policy}).

\color{black}
\vspace{-1mm}
\floatname{algorithm}{Alg.}
\begin{mdframed}
\vspace*{-5mm}
\begin{algorithm}[H] 
\scriptsize
\caption{ Value iteration over product automaton}\label{Algo}
\setcounter{AlgoLine}{0}
\SetKwInOut{Input}{Input}
\SetKwInOut{Output}{Output}
\Input{product automaton $\mathcal{P}$, reward function $r$, discount factor $\gamma$, convergence threshold $\epsilon_0$}
\Output{Optimal policy $\pi_p^*$}

\For{\text{All} $s_p \in S_p$}{
      $v(s_p) = 0$\;
    }
    
    $\epsilon = \infty$
    
    \While{$\epsilon > \epsilon_0$}{

       $\epsilon = 0$
       
      \For{\text{each} $s_p \in S_p$}{

        \hspace{-2mm}$\text{old-value} = v(s_p)$
        
        \scriptsize$\hspace{-2mm}v(s_p) = \hspace{-.5mm}\max\limits_{\substack{\sigma}}\hspace{-2mm} \sum\limits_{\substack{{s_p^\prime \in S_p}}}\hspace{-1mm}p_{p}(s_p,\sigma,s_p^\prime)(r(s_p,\sigma, s_p^\prime)\hspace{-0.5mm}+\hspace{-1mm}\gamma v(s_p^\prime))$\;
        
        \hspace{-2mm}$\epsilon = \max(\epsilon, |v(s_p) - \text{old-value}|)$
      }
      
    }\For{\text{each} $s_p \in S_p$}{
        $\hspace{-2mm}v^*(s_p) = v(s_p)$\;
        
        \scriptsize$\hspace{-2mm}\pi_p^*(s_p)\hspace{-0.25mm} = \argmax\limits_{\substack{\sigma}}\hspace{-2mm}\sum\limits_{\substack{{s_p^\prime \in S_p}}}\hspace{-1mm}p_{p}(s_p,\sigma,s_p^\prime)(r(s_p,\sigma,s_p^\prime)\hspace{-0.5mm}+\hspace{-0.5mm}\gamma v^*(s_p^\prime))$\;
      }

\end{algorithm}
\vspace{-5mm}
\end{mdframed}

\vspace{-1mm}
To integrate new information gathered by the robot during the mission, Alg. 1 should be executed iteratively, for example, after each action. However, updating the policy at every step can be computationally expensive and often unnecessary. To determine when an update is necessary, we first need to define the information matrix \textcolor{black}{$I(N_{x_i}^h, \mathcal{M}) \in \mathbb{R}^{m \times |2^O|}$, where $\mathcal{M}$ is the PL-DMDP, $N_{x_i}^h \subseteq X$ is the set of $h$-hop neighbor states of the current state $x_i \in X$, $m$ \textcolor{black}{is} the cardinality of $N_{x_i}^h$, and each element is defined as follows:}  

\begin{equation}
\scriptsize
I(N_{x_i}^h, \mathcal{M}) = \begin{pmatrix}
I_{11} & I_{12} & \cdots & I_{1|2^O|} \\
I_{21} & I_{22} & \cdots & I_{2|2^O|} \\
\vdots & \vdots & \ddots & \vdots \\
I_{m1} & I_{m2} & \cdots & I_{m|2^O|}
\end{pmatrix},
\end{equation}
\noindent where $I_{jk}$ represents the difference between the true probability of $ l_k$ in state $x_j$ and the prior belief as per $\mathcal{M}$, i.e.,
\begin{equation}\label{eq:information_matrix}
    \small
    I_{jk}
    =
    \begin{Bmatrix}
    |1- p_L(  x_j,  l_k)|& \text{      if }  l_{true}(  x_j) =   l_k   \\
     |0 - p_L(  x_j,  l_k)|& \text{    otherwise } 
    \end{Bmatrix}.
\end{equation}

Note that $p_L$ here refers to its value at the step prior to entering the current state $x_i$.
\color{black}
\vspace{-1mm}
\begin{definition}[Update-trigger function]
    \textit{ The update-trigger function calculates the infinity norm of the information matrix $I(N_{x_i}^h,\mathcal{M})$, which is $\|I(N_{x_i}^h,\mathcal{M})\|_{\infty}$.} 
\label{update_trigger}
\end{definition}

According to \eqref{eq:information_matrix}, a non-zero value of $\|I(N_{x_i}^h,\mathcal{M})\|_{\infty}$ indicates a difference between the robot's previous belief and the true environment, which necessitates an update.

\subsection{\color{black} Main Algorithm\color{black}}  
The main algorithm (Alg. 2) begins by constructing the total DFA $\mathcal{A}$ and the product automaton  $\mathcal{P}$ (lines 1 and 2). Lines 3 to 5 set the initial values of the states. In line 4, $\mathcal{X}$ is the set of uncertain states (i.e., their $l_{true}$ have not been revealed). 
The main loop starts at line 6. Line 7 checks if an update is necessary and ensures the execution of value iteration to find an initial policy at $t=0$. Line 8 updates edge probabilities in $\mathcal{P}$ to incorporate any new information about $l_{true}$ into the policy. Note that the map is updated at each step of the mission (\textcolor{black}{line 10}). The $\textbf{Update-Map}$ is a function that takes the $h$-hop neighbors of the current state $x_{\text{current}}$ and $\mathcal{M}$ as inputs. The output of this function is the PL-DMDP with an updated $p_L$ function, in which the true labels of states within $N_{x_{\text{current}}}^h$ are incorporated as in~\eqref{eq:map_update}. \color{black} \textcolor{black}{Lines 11 to 17} involve executing $\pi_p^*$, updating the current states and the set of uncertain states. Overall, the main algorithm returns a sequence of actions that satisfies the given specification in the PL-DMDP\textcolor{black}{, as shown in Fig. \ref{fig:figure5}.}
\vspace{-1mm}

\begin{mdframed}
\vspace*{-5mm}
\begin{algorithm}[H]
\scriptsize
\caption{Main Algorithm \color{black}
(Online)\color{black}}\label{Algo}
\setcounter{AlgoLine}{0}
\SetKwInOut{Input}{Input}
\SetKwInOut{Output}{Output}
\Input{$\mathcal{M} = (X, \Sigma, \delta, O, L,l_{true}, p_L, c)$, $\phi$, Initial\_state ($x_0$)}
\Output{Sequence of actions $\color{black} \bm{\sigma{\scriptstyle[0:n]}} \color{black}$}
\text{From $\phi$ construct $\mathcal{A} = (S, s_0, 2^O, \delta_{\color{black}a\color{black}}, F_a, \{s_t\})$}\;

\text{From $\mathcal{M}$ and $\mathcal{A}$ construct $\mathcal{P}=(S_p, S_{p_0}, \Sigma, \delta_p, p_p, \mathcal{F}_a, \mathcal{F}_t)$ }\;

$\color{black}\bm{\sigma{\scriptstyle[0:n]}}\color{black} = [\ ] $, $x_\text{current} = x_0$, $s_\text{current}= s_0, {s_{p}}_\text{current} =(x_0,s_0)$ \;

$\mathcal{X} = X \setminus N_{x_\text{current}}^h$

$t = 0$

\While{${s_p}_\text{current} \notin \mathcal{F}_a$}{
  
  \If{$\|I(N_{x_\text{current}}^h,\mathcal{M})\|_{\infty} > 0$ \text{or} $t \hspace{-1mm}=\hspace{-0.5mm} 0$}{

  \text{Update the product automaton based on the most recent $p_L$ }

  \text{Calculate} $\pi_p^*$ \text{ from Alg. 1} \color{black} (Problem 1) \color{black}
  \color{black}

    }

    $\mathcal{M}= \textbf{Update-map}(N_{x_\text{current}}^h,\mathcal{M})$

      \text{Execute $\pi_p^*({s_p}_\text{current})$}

          $ \bm{\sigma} [t] = \pi_p^*({s_p}_\text{current})$

      \color{black}$t = t+1$\color{black}

      $x_\text{current} = \delta(x_\text{current},\pi_p^*({s_p}_\text{current}))$\;
      
      $s_\text{current}= \delta_{\color{black}a\color{black}}(s_\text{current}, l_{true}(x_\text{current}))$
      
      ${s_p}_\text{current} = (x_\text{current}, s_\text{current})$

      $\mathcal{X} = \mathcal{X}\setminus N_{x_\text{current}}^h$

     }
    
    \text{return} $\color{black} \bm{\sigma{\scriptstyle[0:n]}} \color{black}$

\end{algorithm}
\vspace{-5mm}
\end{mdframed}


\begin{lemma}
\label{safety}
\textit{Let Assumptions 1 and 2 hold, and let $x_{current} \in X$ denote the current state. Then, the optimal policy $\pi_p^*$ in  (\ref{eq:policy}) never leads the robot to enter a trash state.}
\end{lemma}
\begin{proof} 

Assumption \ref{sensor} implies that the robot is aware of the true labels $l_{true}$ of its neighbors within $N_{x_\text{current}}^h$ for some $h \geq 1$. 
Consider a non-accepting and non-trash state $s_p\in S_p\setminus {(\mathcal{F}_a\cup \mathcal{F}_t)}$ where there exists an action $\sigma \in \Sigma$ that results in a transition from $s_p$ to a trash state $s_p' \in \mathcal{F}_t$, i.e., $p_p(s_p,\sigma,s_p') = 1$ which is available to the robot due to Assumption \ref{sensor}. Assume that the optimal policy $\pi_p^*$ \color{black} selects action $\sigma$ in state $s_p$, i.e., $\pi_p^*(s_p) = \sigma$, which results in a transition to $s_p' \in \mathcal{F}_t$. According to the reward definition in~\eqref{eq:reward}, the expected return from $s_p$ under $\pi_p^*$ is:
$
U^{\pi_p^*}(s_p) =
\frac{-\beta}{1 - \gamma}.
$
However, Corollary~\ref{corollary} guarantees that for any $s_p \in \mathcal{S}_p \setminus (\mathcal{F}_a \cup \mathcal{F}_t)$, it holds that $U^{\pi_p^*}(s_p) > \frac{-\beta}{1 - \gamma}$. This contradicts the outcome $U^{\pi_p^*}(s_p) = \frac{-\beta}{1 - \gamma}$.
Therefore, $\pi_p^*$ cannot lead the robot into a trash state.
\end{proof}

\color{black}

\begin{figure}[!t]
  \centering
  \includegraphics[width=0.48\textwidth,height=0.25\textwidth]{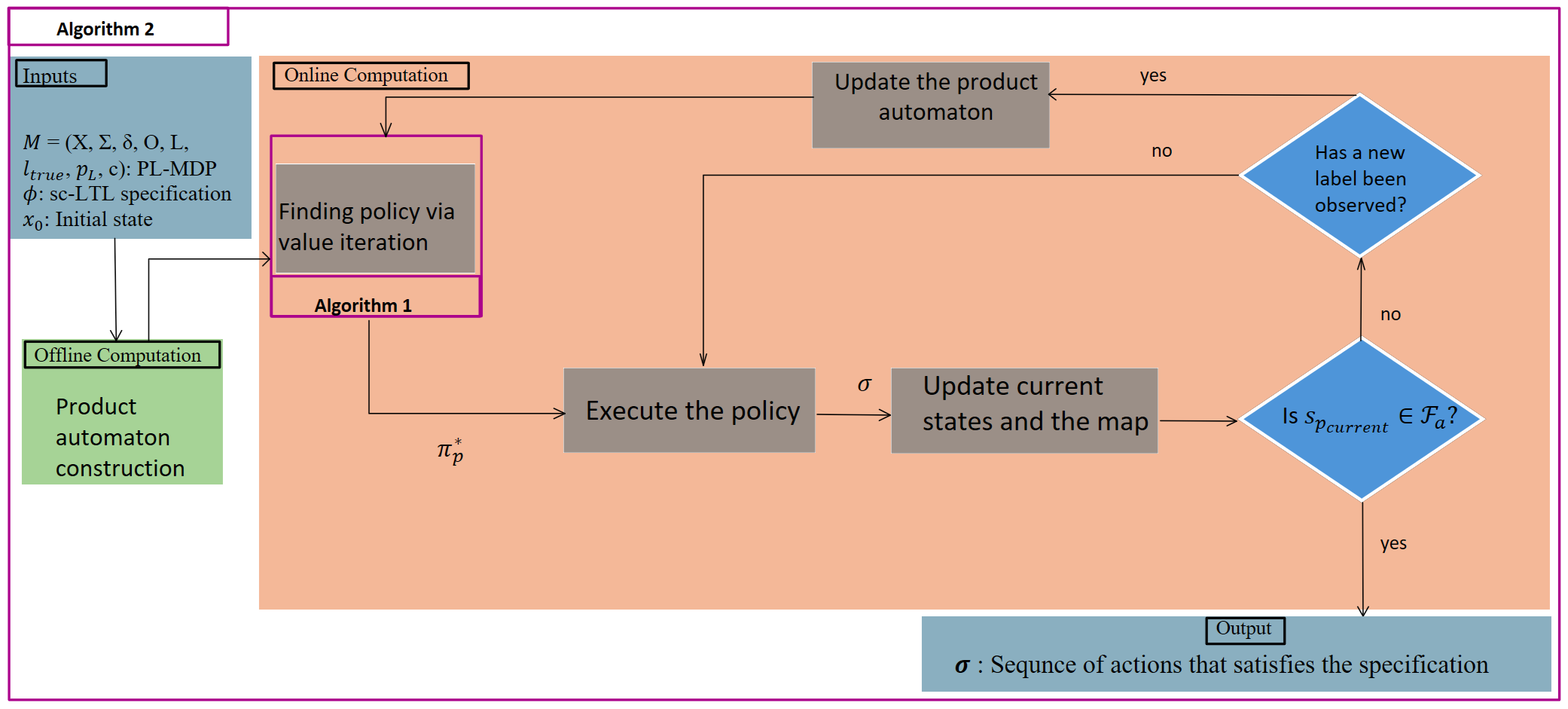}
  \caption{\scriptsize Outline of the proposed framework.}
  \label{fig:figure5}
\end{figure}

\begin{lemma}
\label{lemfortheorem}
\textit{Let Assumptions 1 and 2 hold. Suppose that there exist some uncertain states in the environment, i.e., $\mathcal{X} \neq \emptyset$. Following the policy $\pi_p^*$ results in at least one of the following:}
    \begin{itemize}
        \item the size of the set of uncertain states $\mathcal{X}$ decreases;
        \item the robot reaches an accepting state in the set $\mathcal{F}_a$.
    \end{itemize} 
\end{lemma}
\begin{proof}
Since Assumptions \ref{sensor} and \ref{non-zerp} hold, Corollary~\ref{corollary} implies that $\pi_p^*$ is a non-zero probability satisfying policy. This means that among all the possible trajectories that $\pi_p^*$ can generate, there exists at least one trajectory that reaches an accepting state. \color{black}Additionally, Lemma~\ref{safety} indicates that the robot cannot end up in a trash state by following $\pi_p^*$. Now, let $ \bm{s_p^{\pi_p^*}}= s_p(0)s_p(1) \dots$ be the product automaton trajectory generated under $\pi_p^*$ where each $s_p(i)\in S_p$  for all $i\geq 0$. If for any state $s_p(i) = (x_i,s_i)$ we have $x_i\notin \mathcal{X}$, then $l_{true}(x_i)$ is known for all $i\geq 0$. This implies that $\pi_p^*$ can generate only one possible trajectory $\bm{s_p^{\pi_p^*}}$ over the product automaton. Thus, there must exist a state $s_p(j)$ in $\bm{s_p^{\pi_p^*}}$, where $j\geq 0$, such that $s_p(j)\in \mathcal{F}_a$. On the other hand, if there exists a state $s_p(i) = (x_i,s_i)$ in $\bm{s_p^{\pi_p^*}}$ where $x_i\in \mathcal{X}$, then the robot eventually reveals $l_{true}(x_i)$ by following $\pi_p^*$. This means the size of $\mathcal{X}$ decreases. 
As a result, at least one of the outcomes described in Lemma \ref{lemfortheorem} is attained.\end{proof}

\vspace{-2mm}
\color{black}
\begin{theorem}
\textit{If Assumptions 1 and 2 hold, then Alg. \ref{Algo} will find a sequence of actions that satisfies the scLTL task $\phi$.}\end{theorem}

\begin{proof}
 In Alg. \ref{Algo}, the robot follows the policy $\pi_p^*$. Since Assumptions~\ref{sensor} and~\ref{non-zerp} hold, Lemma~\ref{threshold} indicates that $\pi_p^*$ is a non-zero probability satisfying policy. \color{black}
 Here, there are two possibilities to consider: $\mathcal{X}= \emptyset$ or $\mathcal{X}\neq \emptyset$. 
 \vspace{1mm}
 \paragraph{$\mathcal{X} = \emptyset$} $l_{true}(x)$ for all $x\in X$ is known, which results in no discrepancies between the robot's belief and the true environment. So $\|I(N_{x_\text{current}}^h,\mathcal{M})\|_{\infty}$ is always zero, i.e, line 7 is not triggered and $\pi_p^*$ does not change. Consequently, $\pi_p^*$ will lead the robot to an accepting state (Lemma \ref{known}).
 \vspace{1mm}
 \paragraph{$\mathcal{X} \neq \emptyset$}
 If $|I(N_{x_\text{current}}^h, \mathcal{M})|_{\infty}$ is always zero during the execution of $\pi_p^*$, then the policy $\pi_p^*$ remains unchanged. Since Assumptions 1 and 2 hold, Lemma \ref{lemfortheorem} indicates that following $\pi_p^*$ results in either a decrease in the size of $\mathcal{X}$ or reaching an accepting state. However, if  $\|I(N_{x_\text{current}}^h,\mathcal{M})\|_{\infty} \neq 0$, Def. \ref{update_trigger} indicates that there is a difference between the robot's belief and the true environment. This suggests that the robot has revealed the true labels of some states $x \in N_{x_{\text{current}}}^h$ for which $l_{\text{true}}(x)$ was previously unknown. These states are then removed from $\mathcal{X}$ \textcolor{black}{(line 17)}. 
 
Overall, the robot either enters an accepting state or the size of $\mathcal{X}$ decreases.
Since the PL-DMDP $\mathcal{M}$ consists of a finite number of states, the set of uncertain states will eventually become empty (i.e., $\mathcal{X} = \emptyset$). Accordingly, when $\mathcal{X} = \emptyset$,  Lemma \ref{known} indicates that Alg. \ref{Algo} which computes the policy according to \eqref{eq:policy} leads the robot to an accepting state in $\mathcal{F}_a$. According to Def. \ref{DFA}, reaching a state in $\mathcal{F}_a$ implies the generation of a “good prefix". Therefore, if $\color{black} \color{black} \color{black} \color{black} \bm{\sigma{\scriptstyle[0:n]}} \color{black} \color{black} \color{black} \color{black}$ is the output of Alg. \ref{Algo} and $\color{black}\color{black}\color{black}\bm{x^{\sigma{\tiny[0:n]}}} \color{black} \color{black} \color{black}$ is the trajectory produced by it, then we have $\bm{  l({\color{black}\color{black}\color{black}\bm{x^{\sigma{\tiny[0:n]}}} \color{black} \color{black} \color{black}})}\in \mathscr{L}_{pref,\phi}$. This means the task is satisfied.
 \end{proof}
\color{black}
\color{black}

\vspace{-3mm}
\section{Case Studies}\label{sec:case_studies}

\noindent \textbf{Simulation:} \textcolor{black}{ In the following case study, a robot operates over a $5 \times 5$ grid (Fig. \ref{fig:case1}, with each cell being $2\times 2$ meters)} using an action set $\Sigma = \{Up, Right, Down, Left, Stay\}$. These actions allow the robot to move to the corresponding feasible adjacent state in the four cardinal directions or remain in its current position. 
In our simulations, we also consider that the robot's sensor range $h=1$ and the discount factor $\gamma = 0.99$. \color{black} The computations are carried out on a laptop with an Intel Core i7 processor (2.3 GHz) and 16 GB of RAM.\color{black}

 \color{black}
 \subsubsection{Ordering Reachability with Avoidance}

  In this scenario, the robot must reach the desired regions $A$, $B$, and $C$ while avoiding region $D$ throughout the mission. Visiting $A$ can happen at any time, but $B$ must be visited before $C$. This task is expressed by the scLTL formula $\phi_1 = (\neg C\hspace{3pt} \mathcal{U}\hspace{3pt} B)\land (\lozenge C) \land (\lozenge A) \land (\neg D \hspace{3pt}\mathcal{U} \hspace{3pt}A) \land (\neg D \hspace{3pt}\mathcal{U}\hspace{3pt} C)$. The robot is not aware of the exact locations for $A,B,C  \text{ and } D$ but has access to some information regarding potential locations, as shown in Fig. \ref{fig:case1}. As observed, the robot moves through the environment and follows its policy while updating its map knowledge. Ultimately, the robot successfully finds the real $A,B,C  \text{ and } D$ locations and completes the mission. This scenario is illustrated in Fig. \ref{fig:case1}, where yellow cells represent potential A locations, blue cells indicate potential B locations, green cells denote potential C locations, and red cells mark possible D locations.
  

\begin{figure}[!htbp]
    \centering
    \begin{tabular}{ccc} 
        \begin{subfigure}[t]{0.25\columnwidth}
            \centering
            \includegraphics[width=\textwidth]{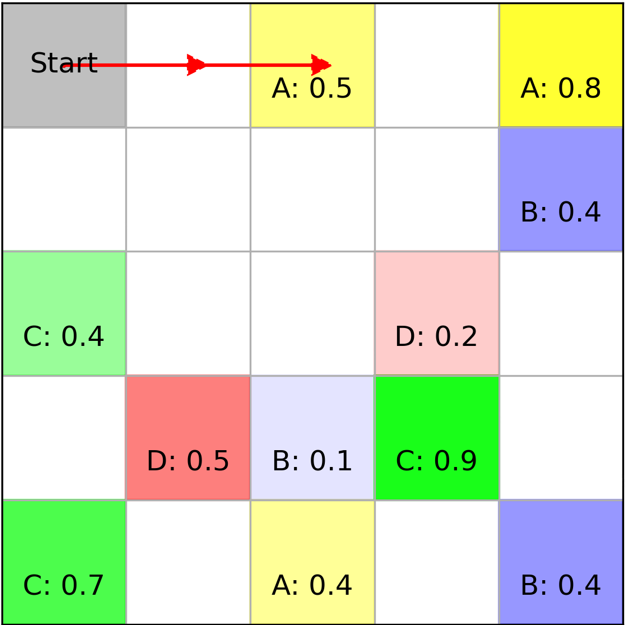}
            \caption{}
        \end{subfigure} &
        \begin{subfigure}[t]{0.25\columnwidth}
            \centering
            \includegraphics[width=\textwidth]{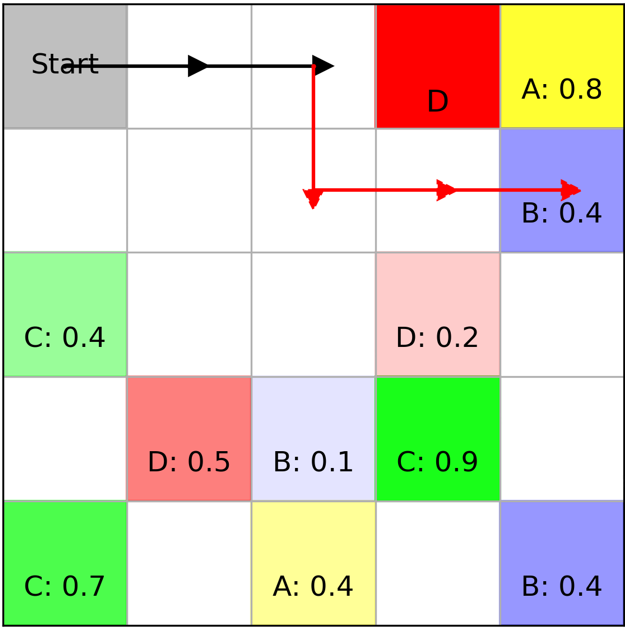}
            \caption{}
        \end{subfigure} &
        \begin{subfigure}[t]{0.25\columnwidth}
            \centering
            \includegraphics[width=\textwidth]{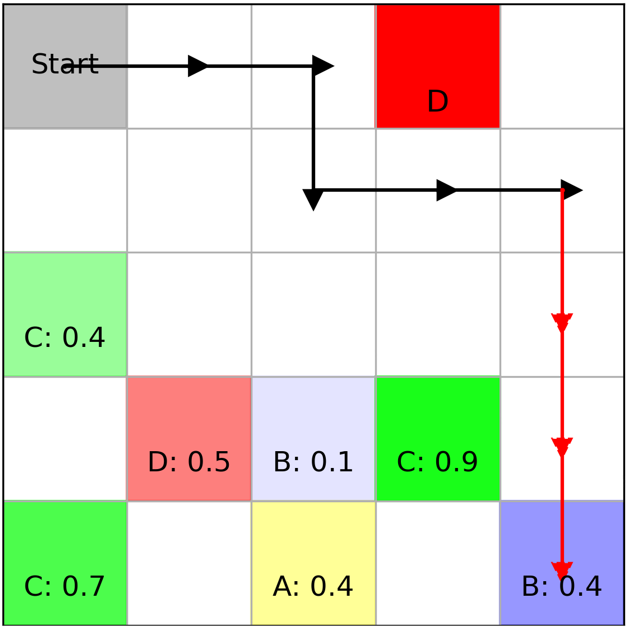}
            \caption{}
        \end{subfigure} \\
        \begin{subfigure}[t]{0.25\columnwidth}
            \centering
            \includegraphics[width=\textwidth]{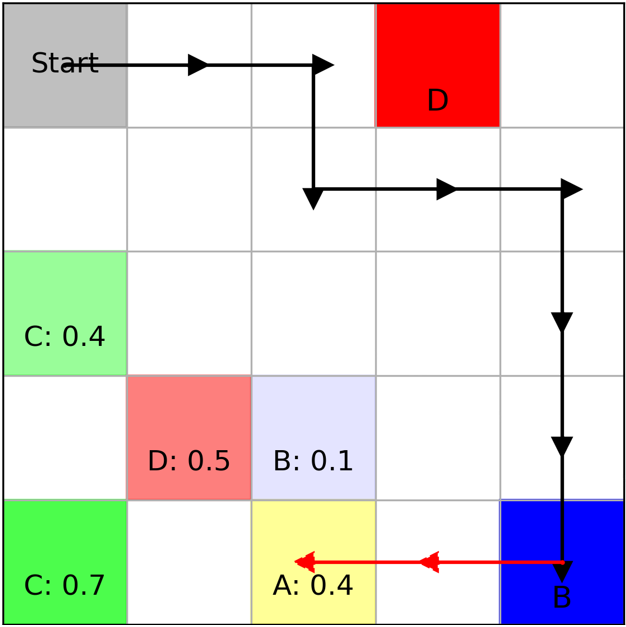}
            \caption{}
        \end{subfigure} &
        \begin{subfigure}[t]{0.25\columnwidth}
            \centering
            \includegraphics[width=\textwidth]{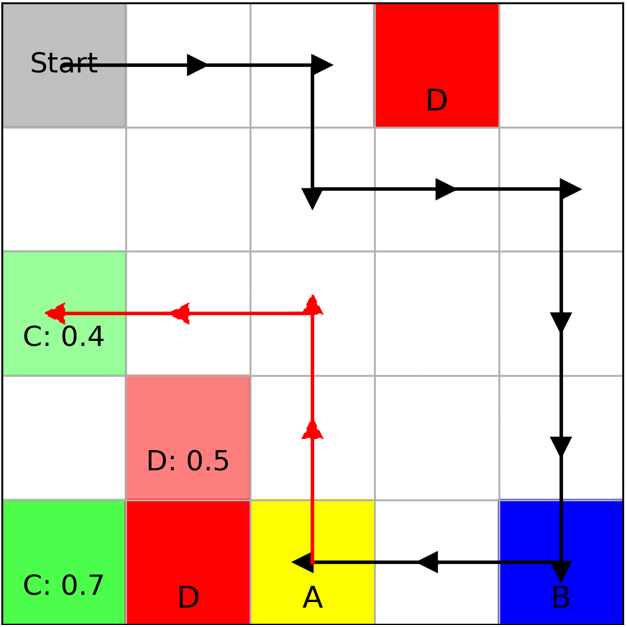}
            \caption{}
        \end{subfigure} &
        \begin{subfigure}[t]{0.25\columnwidth}
            \centering
            \includegraphics[width=\textwidth]{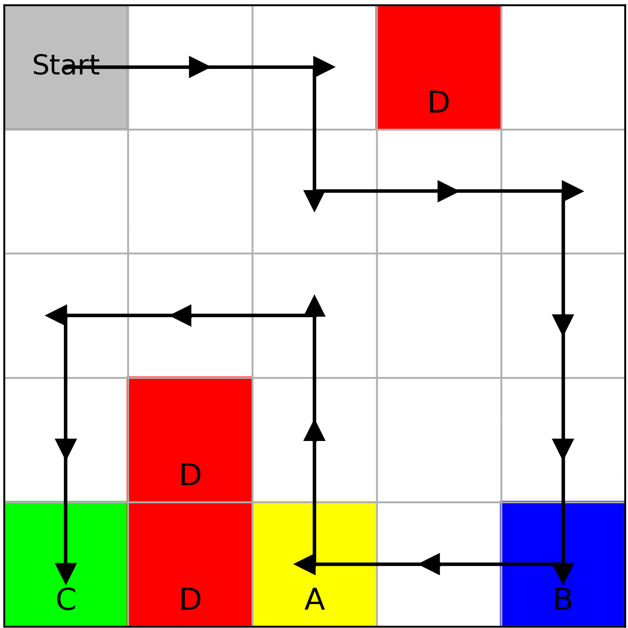}
            \caption{}
        \end{subfigure}
        
    \end{tabular}
    \caption{\scriptsize The illustrations of (a) the planned trajectory in \emph{red} based on the initial belief, (b,c,d,e) the re-planned trajectories in \emph{red} based on the current belief and the realized trajectory in \emph{black}, (f) the final trajectory accomplishing the mission.}
\label{fig:case1}    
\end{figure}

 
\color{black}
\subsubsection{Scalability}
We evaluate the scalability of the proposed algorithm by computing a policy to satisfy the scLTL task $\phi_2 = \lozenge(A \land \lozenge (B \land \lozenge (C)))$ (which implies visiting $A$, $B$, and $C$ in order) across different size environments. \textcolor{black}{TABLE}~\ref{tab:split_column_example} presents the time required for the product automaton construction and the policy calculation across various sizes of PL-DMDPs. As PL-DMDP size increases, the computation time grows, but only the policy computation occurs online. 
        
\begin{table}[!htbp]
\color{black}
\small
    \centering
    \resizebox{0.8\columnwidth}{!}{%
    \setlength{\tabcolsep}{4pt} 
    \renewcommand{\arraystretch}{1} 
    \begin{tabular}{|>{\centering\arraybackslash}p{2cm}|>{\centering\arraybackslash}p{2cm}|>{\centering\arraybackslash}p{1.3cm}|>{\centering\arraybackslash}p{1.3cm}|}
        \hline
        $\mathcal{M}$ & \multicolumn{2}{|c|}{$\mathcal{P}$} & $\pi_p$ \\ \hline
          Size        &  Size & Time [s]      &  Time [s]        \\ \hline

        $(100,460)$    & $(400, 4.6e3)$             & $0.12$ & $0.17$ \\ \hline
        $(400,1.92e3)$    & $(1.6e3, 1.92e4)$             & $1.76$  & $1.43$ \\ \hline
         $(900,4.38e3)$    & $(3.6e3, 4.38e4)$             & $7.22$  & $5.27$ \\ \hline
        $(1.6e3,7.84e3)$    & $(6.4e3, 7.84e4)$           & $30.04$ & $13.35$ \\ \hline
        $(2.5e3,1.23e4)$    & $(1e4, 1.23e5)$           & $67.69$ & $28.11$ \\ \hline
         
    \end{tabular}
    }
    \caption{\color{black}
Computation times for the product automaton construction and policy calculation under the same task $\phi_2$, evaluated for different sizes of $\mathcal{M}$ with $\gamma = 0.99$ and $\epsilon_0 = 0.01$.}
    \label{tab:split_column_example}
\end{table}
\color{black}
\subsubsection{Benchmark analysis}
\paragraph{Trajectory Performance}\color{black}
\textcolor{black}{We compare the performance of our algorithm against two other strategies in terms of the trajectory lengths required to satisfy the desired task $\phi_3 = \lozenge (\text{Pick-up} \land (\lozenge \text{Delivery}))$.}  \color{black}
\noindent \textbf{Map generation -} We conduct Monte Carlo simulations to generate 500 $6 \times 6$ true maps based on a fixed map with initial belief (as in Fig.~\ref{fig:monte_initial}). The process to generate a true map is as follows: a random number between $[0,1]$ is generated  to decide on the existence of a pick-up and/or delivery location according to the initial belief (yellow/blue locations and the corresponding \color{black}probability \color{black} thresholds). After this process, if there is not at least one pick-up and one delivery location, we randomly select one state from the set of possible pick-up locations and another from the set of possible delivery locations. These states are then assigned as pick-up and delivery, respectively, to satisfy Assumption \ref{non-zerp}. \color{black} 
\color{black}\noindent \textbf{Benchmark strategies -}
The \textcolor{black}{first} strategy is the offline method in \cite{8272366}. The \textcolor{black}{second} strategy builds upon \cite{8272366} by incorporating belief updates and replanning, a feature not present in \cite{8272366}. To ensure consistency with our scenarios, we make minor modifications to their model, such as using the same action sets and deterministic transitions. Note that both our method and \cite{8272366} use the same initial belief (as in Fig.~\ref{fig:monte_initial}).
{\noindent \textbf{Results -} 
\textcolor{black}{As shown in TABLE \ref{tab:statistics}, 
our algorithm completes the mission with an average trajectory length of $9.75$. Although the method in \cite{8272366} achieves an average trajectory length of $6$, it satisfies the task in only $139$ maps due to deadlocks, as explained in Section \Romannum{2}. Even after incorporating belief updates into \cite{8272366}, the modified approach results in a higher average trajectory length of $10.67$. 
\paragraph{Runtime Performance}
 We also compare our method against the approach in~\cite{8272366} with belief update, in terms of product automaton construction time and policy calculation time. We denote by $\mathcal{B}$ the percentage of states in $\mathcal{M}$ that have uncertain labels in the initial belief. For each such state $x \in X$, the initial belief is assigned as follows: $p_L(x,\{A\}) = p_L(x,\{B\}) = p_L(x,\{C\}) = p_L(x,\{A,B\}) = p_L(x,\{A,C\}) = p_L(x,\{B,C\}) = p_L(x,\{A,B,C\})= 0.1$, and $p_L(x,\{\}) = 0.3$. We conduct this comparison across varying values of $\mathcal{B}$ and different sizes of $\mathcal{M}$, using the task $\phi_4 = (\lozenge A) \land (\lozenge B) \land (\lozenge C)$ under the same initial belief  for both methods. As shown in TABLE \ref{table-time}, our method outperforms \cite{8272366} with belief update by achieving faster policy computation in all cases and faster product automaton construction in most cases. While our computation times show little variation with increasing $\mathcal{B}$, the approach in~\cite{8272366} becomes substantially slower as $\mathcal{B}$ grows. Note that this difference in computation time becomes significant when planning in large unknown environments, where the number of replanning steps may be on the order of hundreds.
 \color{black}
Overall, our proposed method strategically utilizes time-varying knowledge about the environment and provides efficient solutions to motion planning problems in semantically uncertain environments. 
}

}\color{black}

\begin{table}[!htbp]
    \centering
    \resizebox{0.85\columnwidth}{!}{%
    \setlength{\tabcolsep}{4pt} 
    \renewcommand{\arraystretch}{1.1} 
    \begin{tabular}{|>{\centering\arraybackslash}p{3.5cm}|>{\centering\arraybackslash}p{1cm}|>{\centering\arraybackslash}p{1.1cm}|>{\centering\arraybackslash}p{1.1cm}|>{\centering\arraybackslash}p{1.1cm}|>{\centering\arraybackslash}p{1.5cm}|>{\centering\arraybackslash}p{1.1cm}|}
        \hline
        \textbf{Method} & \textbf{Mean} & \textbf{Median} & \textbf{SD} & \textbf{Success} \\ \hline
        Our Algorithm            & 9.75  & 9.0  & 3.73   & 500 \\ \hline

        {\cite{8272366}}  & 6  & 6  & 0    & 139 \\ \hline
        {\cite{8272366}} + belief update & 10.67 & 8.5 & 5.43 & 500 \\ \hline
    \end{tabular}
    }
    \caption{\color{black} Comparison of our algorithm,  \cite{8272366} and \cite{8272366} with belief update based on trajectory length statistics (mean, median, and standard deviation) and mission success count across 500 maps. }
    \label{tab:statistics}
\end{table}

\begin{table}[!htbp]
\centering
\captionsetup{justification=centering}
\begin{subtable}[t]{\linewidth}
\centering
\resizebox{0.98\columnwidth}{!}{%
\setlength{\tabcolsep}{4pt}
\renewcommand{\arraystretch}{1.2}
\begin{tabular}{|>{\centering\arraybackslash}p{2.5cm}|
                >{\centering\arraybackslash}p{1.5cm}|
                >{\centering\arraybackslash}p{3cm}|
                >{\centering\arraybackslash}p{3cm}|
                >{\centering\arraybackslash}p{1.8cm}|
                >{\centering\arraybackslash}p{1.8cm}|}
\hline
$\boldsymbol{\mathcal{M}}$ (Size) & $\boldsymbol{\mathcal{B}}$ & \textbf{Method} & $\boldsymbol{\mathcal{P}}$ (Size) & $\mathcal{P}$ Time [s] & $\pi_p$ Time [s] \\ \hline

\multirow{8}{*}{(100,\ 460)} 
  & \multirow{2}{*}{25\%}  & Our algorithm           & $(800,\ 1.242\text{e}4)$   & $0.27$ & $0.24$ \\ \cline{3-6}
                           &                         & \cite{8272366} + belief update       & $(2.2\text{e}3,\ 3.689\text{e}4)$ & $0.39$ & $0.81$ \\ \cline{2-6}
  & \multirow{2}{*}{50\%}  & Our algorithm           & $(800,\ 1.242\text{e}4)$   & $0.27$ & $0.21$ \\ \cline{3-6}
                           &                         & \cite{8272366} + belief update       & $(3.6\text{e}3,\ 7.894\text{e}4)$ & $0.85$ & $1.77$ \\ \cline{2-6}
  & \multirow{2}{*}{75\%}  & Our algorithm           & $(800,\ 1.242\text{e}4)$   & $0.27$ & $0.19$ \\ \cline{3-6}
                           &                         & \cite{8272366} + belief update       & $(5\text{e}3,\ 1.507\text{e}5)$   & $1.65$ & $4.54$ \\ \cline{2-6}
  & \multirow{2}{*}{100\%} & Our algorithm           & $(800,\ 1.242\text{e}4)$   & $0.27$ & $0.18$ \\ \cline{3-6}
                           &                         & \cite{8272366} + belief update       & $(6.4\text{e}3,\ 2.355\text{e}5)$ & $2.47$ & $5.81$ \\ \hline
\end{tabular}
}
\caption{}
\end{subtable}

\vspace{0.3cm}

\begin{subtable}[!htp]{\linewidth}
\centering
\resizebox{0.98\columnwidth}{!}{%
\setlength{\tabcolsep}{4pt}
\renewcommand{\arraystretch}{1.2}
\begin{tabular}{|>{\centering\arraybackslash}p{2.5cm}|
                >{\centering\arraybackslash}p{1.5cm}|
                >{\centering\arraybackslash}p{3cm}|
                >{\centering\arraybackslash}p{3cm}|
                >{\centering\arraybackslash}p{1.8cm}|
                >{\centering\arraybackslash}p{1.8cm}|}
\hline
$\boldsymbol{\mathcal{M}}$ (Size) & $\boldsymbol{\mathcal{B}}$ & \textbf{Method} & $\boldsymbol{\mathcal{P}}$ (Size) & $\mathcal{P}$ Time [s] & $\pi_p$ Time [s] \\ \hline

\multirow{8}{*}{(400,\ 1.92e3)} 
  & \multirow{2}{*}{25\%}  & Our algorithm           & $(3.2\text{e}3,\ 5.184\text{e}4)$   & $3.83$ & $1.05$ \\ \cline{3-6}
                           &                         & \cite{8272366} + belief update       & $(8.8\text{e}3,\ 1.433\text{e}5)$ & $1.57$ & $3.53$ \\ \cline{2-6}
  & \multirow{2}{*}{50\%}  & Our algorithm           & $(3.2\text{e}3,\ 5.184\text{e}4)$   & $3.83$ & $1.00$ \\ \cline{3-6}
                           &                         & \cite{8272366} + belief update       & $(1.44\text{e}4,\ 3.464\text{e}5)$ & $3.66$ & $8.10$ \\ \cline{2-6}
  & \multirow{2}{*}{75\%}  & Our algorithm           & $(3.2\text{e}3,\ 5.184\text{e}4)$   & $3.85$ & $0.98$ \\ \cline{3-6}
                           &                         & \cite{8272366} + belief update       & $(2\text{e}4,\ 6.290\text{e}5)$     & $6.66$ & $20.3$ \\ \cline{2-6}
  & \multirow{2}{*}{100\%} & Our algorithm           & $(3.2\text{e}3,\ 5.184\text{e}4)$   & $3.85$ & $0.83$ \\ \cline{3-6}
                           &                         & \cite{8272366} + belief update       & $(2.56\text{e}4,\ 9.830\text{e}5)$ & $10.47$ & $44.40$ \\ \hline
\end{tabular}
}
\caption{}
\end{subtable}

\caption{Computation time comparison between our algorithm and \cite{8272366} with belief update, which shows product automaton construction time and policy computation times under task $\phi_4$ for (a) $\mathcal{M} = (100,\ 460)$ and (b) $\mathcal{M} = (400,\ 1.92\text{e}3)$, across varying $\mathcal{B}$ values.}
\label{table-time}
\end{table}

\begin{figure}[!htbp]
    \centering
    \includegraphics[width=0.2\textwidth]{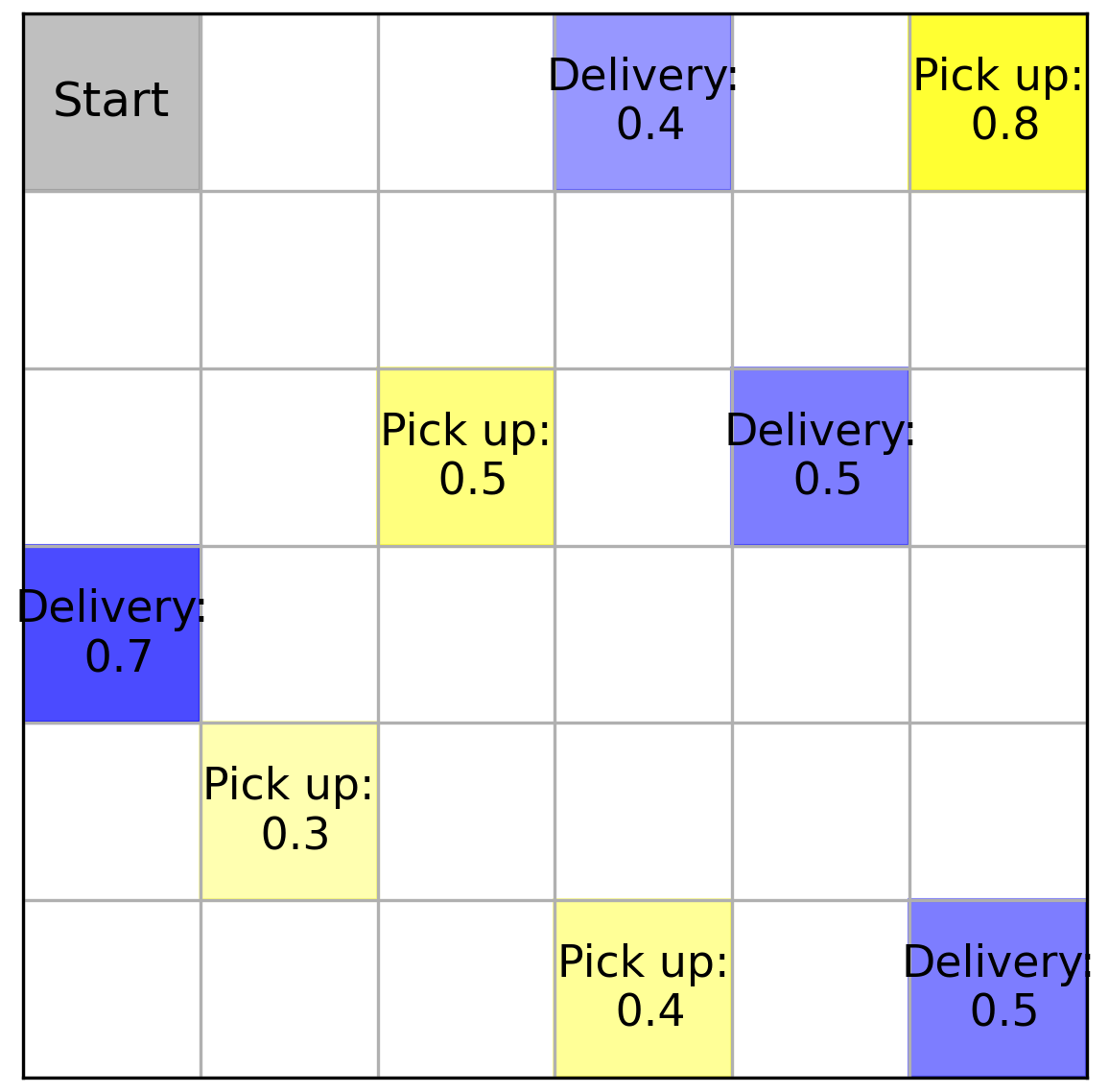}
    \caption{\scriptsize  A 6 by 6 grid map with initial information about the possible location of labels.}
    \label{fig:monte_initial}
\end{figure}

\color{black}

\noindent \textbf{Experiments:} We test  our algorithm for real-time planning with a DJI Tello quadrotor in scenarios such as pick-up/delivery and reach/avoid. We use the DJITelloPy \cite{djitellopy2024} package for low-level controls. We utilize AprilTags to accurately detect different labels like $Pickup$, $Delivery$, $Goal$, and $Danger$. This enables the quadrotor to update its map effectively by recognizing these tags during operation. \color{black} Additionally, we conduct a Gazebo simulation with the Scout Mini robot, which tackles a more complex task in a larger environment. \color{black} More details about the experiments can be found at the following link: \color{black} (\url{https://youtu.be/CfB4e2B5P8Q})\color{black}.
\color{black}

\section{Conclusion}\label{sec:conclusion}
We introduce an automata-theoretic framework to solve motion planning problems under syntactically co-safe Linear Temporal Logic (scLTL) specifications in environments with uncertain semantic maps. Our method utilizes value iteration techniques to generate a sequence of actions that satisfies the desired scLTL task. \color{black}By updating beliefs dynamically based on new information, our approach ensures adaptive and accurate planning. Our key contributions include a compact representation of the product automaton and the ability to replan based on newly discovered information. 
In future work, we aim to develop a scalable version of this approach to handle larger and more complex environments. \color{black} Additionally, we plan to investigate our problem in the reinforcement learning setting (e.g., \cite{cai2021reinforcement}) and focus on developing novel scalable methods that ensure task satisfaction even when the belief lacks sufficient information, by incorporating safe exploration techniques. \color{black}


\bibliographystyle{IEEEtran}
\bibliography{references}

\end{document}